\def\year{2020}\relax
\documentclass[letterpaper]{article} 
\usepackage{aaai20}  
\usepackage{times}  
\usepackage{helvet} 
\usepackage{courier}  
\usepackage[hyphens]{url}  
\usepackage{graphicx} 
\usepackage{graphicx}  
\usepackage[american]{babel}
\usepackage{microtype}
\usepackage{makecell}
\usepackage{multirow}
\usepackage{booktabs}
\usepackage{amsmath}
\usepackage{amssymb}
\usepackage{subfig}
\allowdisplaybreaks[3]
\newtheorem{thm}{Theorem}
\newenvironment{proof}{{\noindent\it Proof.}\quad}{\hfill $\square$\par}

\title{Repetitive Reprediction Deep Decipher for Semi-Supervised Learning}

\author{
Guo-Hua Wang, Jianxin Wu\thanks{J. Wu is the corresponding author.}\\ 
National Key Laboratory for Novel Software Technology\\ 
Nanjing University\\
Nanjing, China\\
wangguohua@lamda.nju.edu.cn, wujx2001@nju.edu.cn 
}

\begin{document}

\maketitle

\begin{abstract}
Most recent semi-supervised deep learning (deep SSL) methods used a similar paradigm: use network predictions to update pseudo-labels and use pseudo-labels to update network parameters iteratively. However, they lack theoretical support and cannot explain why predictions are good candidates for pseudo-labels. In this paper, we propose a principled end-to-end framework named deep decipher (D2) for SSL. Within the D2 framework, we prove that pseudo-labels are related to network predictions by an exponential link function, which gives a theoretical support for using predictions as pseudo-labels. Furthermore, we demonstrate that updating pseudo-labels by network predictions will make them uncertain. To mitigate this problem, we propose a training strategy called repetitive reprediction (R2). Finally, the proposed R2-D2 method is tested on the large-scale ImageNet dataset and outperforms state-of-the-art methods by 5 percentage points.
\end{abstract}

\section{Introduction}

Deep learning has achieved state-of-the-art results on many visual recognition tasks. However, training these models often needs large-scale datasets such as ImageNet~\cite{imagenet}. Nowadays, it is easy to collect images by search engines, but image annotation is expensive and time-consuming. Semi-supervised learning (SSL) is a paradigm to learn a model with a few labeled data and massive amounts of unlabeled data. With the help of unlabeled data, the model performance may be improved.

With a supervised loss, unlabeled data can be used in training by assigning pseudo-labels to them.
Many state-of-the-art methods on semi-supervised deep learning used pseudo-labels implicitly. Temporal Ensembling~\cite{Temporal_Ensembling} used the moving average of network predictions as pseudo-labels. Mean Teacher~\cite{Mean_teacher} and Deep Co-training~\cite{DCT_2018_ECCV} employed another network to generate pseudo-labels. However, they produced or updated pseudo-labels in ad-hoc manners.
Although these methods worked well in practice, there are few theories to support them. A mystery in deep SSL arises: why can predictions work well as pseudo-labels?

In this paper, we propose an end-to-end framework called deep decipher (D2). Inspired by \cite{PENCIL}, we treat pseudo-labels as variables and update them by back-propagation, which is also learned from data. The D2 framework specifies a well-defined optimization problem, which can be properly interpreted as a maximum likelihood estimation over two set of random variables (the network parameters and the pseudo-labels). With deep decipher, we prove that there exists an exponential relationship between pseudo-labels and network predictions, leading to a theoretical support for using network predictions as pseudo-labels. Then, we further analyze the D2 framework and prove that pseudo-labels will become flat (i.e., their entropy is high) during training and there is an equality constraint bias in it. To mitigate these problems, we propose a simple but effective strategy, repetitive reprediction (R2). The improved D2 framework is named R2-D2 and obtaines state-of-the-art results on several SSL problems.

Our contributions are as follows. 
\begin{itemize}
	\item We propose D2, a deep learning framework that deciphers the relationship between predictions and pseudo-labels. D2 updates pseudo-labels by back-propagation. To the best of our knowledge, D2 is the first deep SSL method that learns pseudo-labels from data end-to-end.
	\item Within D2, we prove that pseudo-labels are exponentially transformed from the predictions. Hence, it is reasonable for previous works to use network predictions as pseudo-labels. Meanwhile, many SSL methods can be considered as special cases of D2 in certain aspects.
	\item To further boost D2's performance, we find some shortcomings of D2. In particular, we prove that pseudo-labels will become flat during the optimization. To mitigate this problem, we propose a simple but effective remedy, R2. We tested the R2-D2 method on ImageNet and it outperforms state-of-the-arts by a large margin. On small-scale datasets like CIFAR-10~\cite{cifar}, R2-D2 also produces state-of-the-art results.
\end{itemize}
\begin{figure*}[t]
	\centering
	\includegraphics[width=1.5\columnwidth]{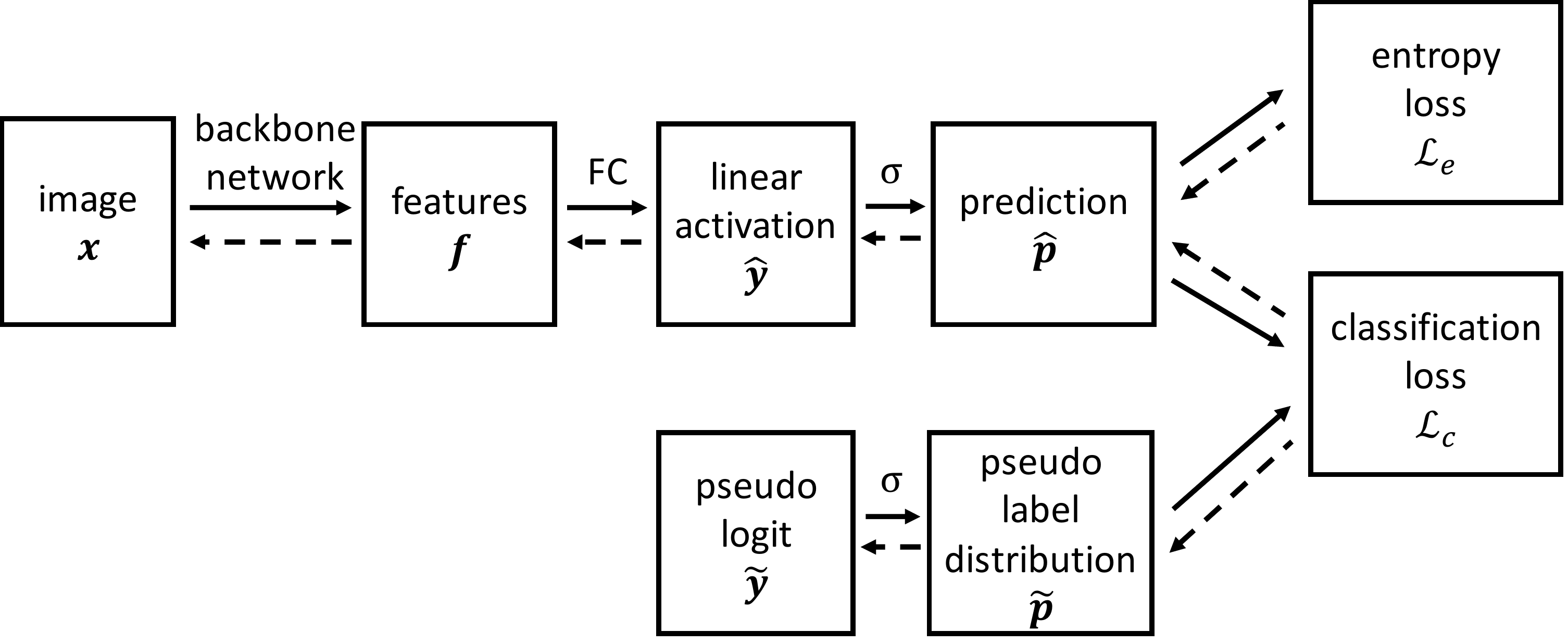}
	\caption{The pipeline of D2. Solid lines and dashed lines represent the forward and back-propagation processes, respectively.}
	\label{fig:framework}
\end{figure*}
\section{Related Works}
We first briefly review deep SSL methods and the related works that inspired this paper.

\cite{pseudo_label} is an early work on training deep SSL models by pseudo-labels, which picks the class with the maximum predicted probability as pseudo-labels for unlabeled images and tested only on a samll-scale dataset MNIST~\cite{LeNet}. Label propagation~\cite{label_prop} can be seen as a form of pseudo-labels. Based on some metric, label propagation pushes the label information of each sample to the near samples. \cite{deep_label_prop} applies label propagation to deep learning models. \cite{cvpr2019_pseudo_label} use the manifold assumption to generate pseudo-labels for unlabeled data. However, their method is complicated and relies on other SSL methods to produce state-of-the-art results.

Several recent state-of-the-art deep SSL methods can be considered as using pseudo-labels implicitly. Temporal ensembling~\cite{Temporal_Ensembling} proposes making the current prediction and the pseudo-labels consistent, where the pseudo-labels take into account the network predictions over multiple previous training epochs. Extending this idea, Mean Teacher~\cite{Mean_teacher} employs a secondary model, which uses the exponential moving average weights to generate pseudo-labels. Virtual Adversarial Training~\cite{VAT} uses network predictions as pseudo-labels, then they want the network predictions under adversarial perturbation to be consistent with pseudo-labels. Deep Co-Training~\cite{DCT_2018_ECCV} employs many networks and uses one network to generate pseudo-labels for training other networks. 

We notice that they all use the network predictions as pseudo-labels but a theory explaining its rationale is missing. With our D2 framework, we demonstrate that pseudo-labels will indeed be related to network predictions. That gives a theoretical support to using network predictions as pseudo-labels. Moreover, pseudo-labels of previous works were designed manually and ad-hoc, but our pseudo-labels are updated by training the end-to-end framework. Many previous SSL methods can also be considered as special cases of the D2 framework in certain aspects.

There are some previous works in other fields that inspired this work.
Deep label distribution learning~\cite{DLDL} inspires us to use label distributions to encode the pseudo-labels. 
\cite{Noisy_Labels} studies the label noise problem. They find it is possible to update the noisy label to make them more precise during the training. PENCIL~\cite{PENCIL} proposes an end-to-end framework to train the network and optimize the noisy labels together. Our method is inspired by PENCIL~\cite{PENCIL}. In addition, inspired by \cite{TCP}, we analyze our algorithm from the gradient perspective.

\section{The R2-D2 Method}

We define the notations first. Column vectors and matrices are denoted in bold (e.g., $\mathbf{x}, \mathbf{X}$). When $\mathbf{x}\in \mathbb{R}^d$,  $x_i$ is the $i$-th element of vector $\mathbf{x}$, $i\in[d]$, where $[d]:=\{1,2,\dots,d\}$. $\mathbf{w}_i$ denote the $i$-th column of matrix $\mathbf{W}\in\mathbb{R}^{d\times l}$, $i\in[l]$. And, we assume the dataset has $N$ classes.

\subsection{Deep decipher}

Figure~\ref{fig:framework} shows the D2 pipeline, which is inspired by \cite{PENCIL}. Given an input image $\mathbf{x}$, D2 can employ any backbone network to generate feature $\mathbf{f}\in\mathbb{R}^D$. Then, the linear activation $\mathbf{\hat{y}}\in \mathbb{R}^N$ is computed as $\mathbf{\hat{y}}=\mathbf{W}^\mathsf{T}\mathbf{f}$, where $\mathbf{W}\in\mathbb{R}^{D\times N}$ are weights of the FC layer and we omit the bias term for simplicity. The softmax function is denoted as $\sigma(\mathbf{y}):\mathbb{R}^N\rightarrow\mathbb{R}^N$ and $\sigma(\mathbf{y})_i=\frac{\exp\left(y_i\right)}{\sum_{j=1}^{N}\exp\left(y_j\right)}$. Then, the prediction $\hat{\mathbf{p}}$ is calculated as $\hat{\mathbf{p}}=\sigma(\mathbf{\hat{y}})$
, hence
\begin{equation}
\hat{p}_n
=\sigma(\mathbf{\hat{y}})_n
=\sigma(\mathbf{W}^\mathsf{T}\mathbf{f})_n
=\frac{\exp(\mathbf{w}_n^\mathsf{T}\mathbf{f})}{\sum_{i=1}^{N}\exp(\mathbf{w}_i^\mathsf{T}\mathbf{f})}\,.
\end{equation}
We define $\tilde{\mathbf{y}}$ as the pseudo logit which is an unconstrained variable and \emph{can} be updated by back-propagation. Then, the pseudo label is calculated as $\tilde{\mathbf{p}} = \sigma(\tilde{\mathbf{y}})$ and it is a probability distribution. 

In the training, the D2 framework is initialized as follows. Firstly, we train the backbone network using only labeled examples, and use this trained network as the backbone network and FC in Figure~\ref{fig:framework}. For labeled examples, $\tilde{\mathbf{y}}$ is initialized by $K\mathbf{y}$, in which $K=10$ and $\mathbf{y}$ is the groundtruth label in the one-hot encoding. Note that $\tilde{\mathbf{y}}$ of labeled examples will \emph{not} be updated during D2 training. For unlabeled examples, we use the trained network to predict $\tilde{\mathbf{y}}$. That means we use the FC layer activation $\hat{\mathbf{y}}$ as the initial value of $\tilde{\mathbf{y}}$. The process of initializing pseudo-labels is called predicting pseudo-labels in this paper.
In the testing, we use the backbone network with FC layer to make predictions and the branch of pseudo-labels is removed.

Our loss function consists of $\mathcal{L}_c$ and $\mathcal{L}_e$. $\mathcal{L}_c$ is the classification loss and defined as $KL(\hat{\mathbf{p}}||\tilde{\mathbf{p}})$ as in \cite{PENCIL}, which is different from the classic KL-loss $ KL(\tilde{\mathbf{p}}||\hat{\mathbf{p}})$. $\mathcal{L}_c$ is used to make the network predictions match the pseudo-labels. $\mathcal{L}_e$ is the entropy loss, defined as $-\sum_{j=1}^{N}\hat{p}_j\log(\hat{p}_j)$. 
Minimizing the entropy of the network prediction can encourage the network to peak at only one category. 
So our loss function is defined as
\begin{align}
\label{loss-function}
	\mathcal{L}
	&=\alpha\mathcal{L}_c + \beta\mathcal{L}_e \notag \\
	&=\alpha\sum_{j=1}^{N}\hat{p}_j\left[\log(\hat{p}_j)-\log(\tilde{p}_j)\right]-\beta\sum_{j=1}^{N}\hat{p}_j\log(\hat{p}_j)\,,
\end{align}
where $\alpha$ and $\beta$ are two hyperparameters. Although there are two hyperparameters in D2, we always set $\alpha=0.1$ and $\beta=0.03$ in all our experiments.

Then, we show that we can decipher the relationship between pseudo-labels and network predictions in D2, as shown by Theorem~\ref{thm:1}.
\begin{thm}
	\label{thm:1}
	Suppose D2 is trained by SGD with the loss function $\mathcal{L}=\alpha\mathcal{L}_c + \beta\mathcal{L}_e$. Let $\hat{\mathbf{p}}$ denote the prediction by the network for one example and $\hat{p}_n$ is the largest value in $\hat{\mathbf{p}}$. After the optimization algorithm converges, we have $\tilde{p}_n\rightarrow\exp(-\frac{\mathcal{L}}{\alpha})\left(\hat{p}_n\right)^{1-\frac{\beta}{\alpha}}$.
\end{thm}

\begin{proof}
	First, the loss function can be rewritten by 
	\begin{align}
	\mathcal{L}
	&=(\alpha-\beta)\sum_{j=1}^{N}\sigma\left(\mathbf{W}^\mathsf{T}\mathbf{f}\right)_j\log\left(\sigma\left(\mathbf{W}^\mathsf{T}\mathbf{f}\right)_j\right) \notag \\
	&\quad
	-\alpha\sum_{j=1}^{N}\sigma\left(\mathbf{W}^\mathsf{T}\mathbf{f}\right)_j\log(\tilde{p}_j)\,.
	\end{align}
	It is easy to see
	\begin{align}
	\frac{\partial\sigma\left(\mathbf{W}^\mathsf{T}\mathbf{f}\right)_j}{\partial\mathbf{w}_n}
	&=\mathbb{I}(j=n)\sigma\left(\mathbf{W}^\mathsf{T}\mathbf{f}\right)_j\mathbf{f} \notag \\
	&\quad
	-\sigma\left(\mathbf{W}^\mathsf{T}\mathbf{f}\right)_j\sigma\left(\mathbf{W}^\mathsf{T}\mathbf{f}\right)_n\mathbf{f}\,.
	\end{align}
	Now we can compute the gradient of $\mathcal{L}$ with respect to $\mathbf{w}_n$:
	\begin{align}
	\frac{\partial\mathcal{L}}{\partial\mathbf{w}_n}
	&=
	\left[(\alpha-\beta)\log\left(\hat{p}_n\right)
	-\alpha\log(\tilde{p}_n)
	-\mathcal{L}
	\right]\hat{p}_n\mathbf{f}\,.
	\end{align}
	
	During training, we expect the optimization algorithm can converge and finally $\frac{\partial\mathcal{L}}{\partial\mathbf{w}_n}\rightarrow \mathbf{0}$. Because $\mathbf{f}$ will not be $\mathbf{0}$, we conclude that $\left[(\alpha-\beta)\log\left(\hat{p}_n\right)
	-\alpha\log(\tilde{p}_n)
	-\mathcal{L}
	\right]\hat{p}_n\rightarrow 0$. Because $\sum_{i=1}^{N}\hat{p}_i=1$, consider the fact that $\hat{p}_n$ is the largest value in $\{\hat{p}_1, \hat{p}_1, \dots, \hat{p}_N\}$, then $\hat{p}_n\not\rightarrow 0$ at the end of training. So we have $\left[(\alpha-\beta)\log\left(\hat{p}_n\right)
	-\alpha\log(\tilde{p}_n)
	-\mathcal{L}
	\right]\rightarrow 0$, which states that $\tilde{p}_n\rightarrow\exp(-\frac{\mathcal{L}}{\alpha})\left(\hat{p}_n\right)^{1-\frac{\beta}{\alpha}}$.
\end{proof}

Theorem~\ref{thm:1} tells us
$\tilde{p}_n$ converges to $\exp(-\frac{\mathcal{L}}{\alpha})\left(\hat{p}_n\right)^{1-\frac{\beta}{\alpha}}$ during the optimization. And at last, we expect that $\tilde{p}_n=\exp(-\frac{\mathcal{L}}{\alpha})\left(\hat{p}_n\right)^{1-\frac{\beta}{\alpha}}$, in which $n$ is the class predicted by the network. 

In other words, we have deciphered that there is an exponential link between pseudo-labels and predictions. From $\tilde{p}_n\rightarrow\exp(-\frac{\mathcal{L}}{\alpha})\left(\hat{p}_n\right)^{1-\frac{\beta}{\alpha}}$, we notice that $\tilde{p}_n$ is approximately proportional to $\hat{p}_n^{1-\frac{\beta}{\alpha}}$. That gives a theoretical support to use network predictions as pseudo-labels. And, it is required that $1-\frac{\beta}{\alpha}>0$ to make pseudo-labels and network predictions consistent. We must set $\alpha>\beta$. In our experiments, if we set $\alpha < \beta$, the training will indeed fail miserably.

Next, we analyze how $\tilde{\mathbf{y}}$ is updated in D2. With the loss function $\mathcal{L}$, the gradients of $\mathcal{L}$ with respect to $\tilde{y}_n$ is
\begin{equation}
\frac{\partial\mathcal{L}}{\partial\tilde{y}_n}
=-\alpha
\sigma\left(\hat{\mathbf{y}}\right)_n
+\alpha\sigma\left(\tilde{\mathbf{y}}\right)_n
\,.
\end{equation}
By gradient descent, the pseudo logit $\tilde{\mathbf{y}}$ is updated by
\begin{equation}
\label{updating-formula}
\tilde{\mathbf{y}}\leftarrow \tilde{\mathbf{y}}-\lambda\frac{\partial\mathcal{L}}{\partial\tilde{\mathbf{y}}}
=\tilde{\mathbf{y}}
-\lambda\alpha\sigma\left(\tilde{\mathbf{y}}\right)
+\lambda\alpha\sigma\left(\hat{\mathbf{y}}\right)\,,
\end{equation}
where $\lambda$ is the learning rate for updating $\tilde{\mathbf{y}}$.
The reason we use one more hyperparameter $\lambda$ rather than the overall learning rate is that $\frac{\partial\mathcal{L}}{\partial\tilde{\mathbf{y}}}=-\alpha\sigma\left(\hat{\mathbf{y}}\right)+\alpha\sigma\left(\tilde{\mathbf{y}}\right)$ is smaller than $\tilde{\mathbf{y}}$ (in part due to the sigmoid transform) and the overall learning rate is too small to update the pseudo logit.
We set $\lambda=4000$ in all our experiments.

The updating formulas in many previous works can be considered as special cases of that of D2. 
In Temporal Ensembling~\cite{Temporal_Ensembling}, the pseudo-labels $\tilde{\mathbf{p}}$ is a moving average of the network predictions $\hat{\mathbf{p}}$ during training. The updating formula is $\mathbf{P}\leftarrow\alpha\mathbf{P}+(1-\alpha)\hat{\mathbf{p}}$. To correct for the startup bias, the $\tilde{\mathbf{p}}$ is needed to be divided by factor $(1-\alpha^t)$, where $t$ is the number of epochs. So the updating formula of $\tilde{\mathbf{p}}$ is $\tilde{\mathbf{p}}\leftarrow \mathbf{P} / (1-\alpha^t)$. In Mean Teacher~\cite{Mean_teacher}, the $\tilde{\mathbf{p}}$ is the prediction of a teacher model which uses the
exponential moving average weights of the student model. \cite{Noisy_Labels} proposed using the running average of the network predictions to estimate the groundtruth of the noisy label. 
However, their updating formula were designed manually and ad-hoc. In contrast, we treat pseudo-labels as random variables like the network parameters. These variables are learned by minimizing a well-defined loss function (cf. equation \ref{loss-function}). From a probabilistic perspective, it is well known that minimizing the KL loss is equivalent to maximum likelihood estimation, in which the backbone network's architecture defines the estimation's functional space while SGD optimizes over these random variables (both the network parameters and the pseudo-labels). We do not need to manually specify how the pseudo-labels are generated. This process is natural and principled.

\subsection{A toy example}

\begin{figure*}[t] 
	\centering 
	\subfloat[]{\label{fig:mnist:sub1}\includegraphics[width=.525\columnwidth]{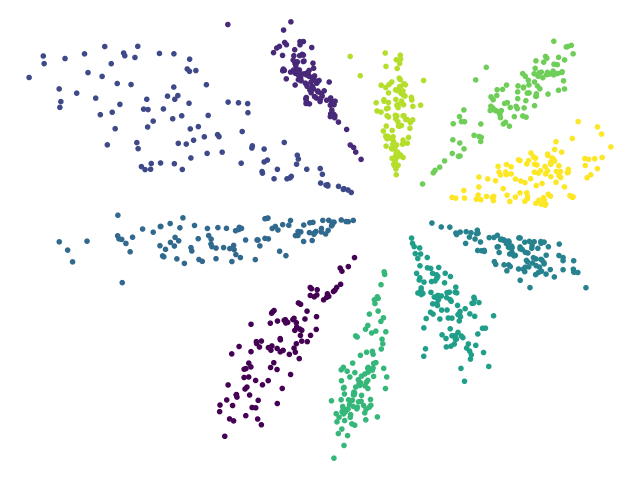}}
	\subfloat[]{\label{fig:mnist:sub2}\includegraphics[width=.525\columnwidth]{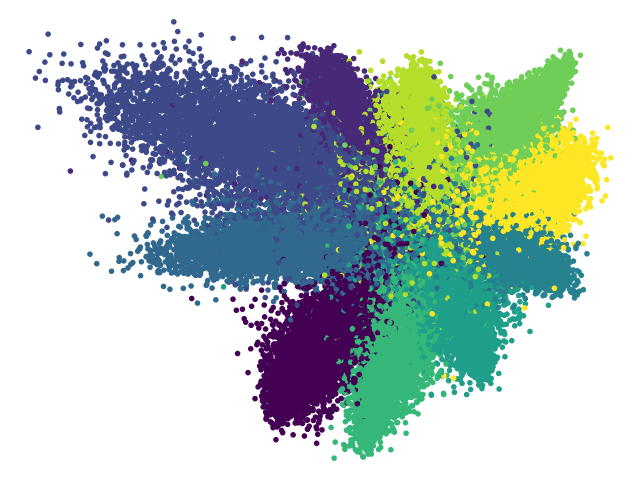}} 
	\subfloat[]{\label{fig:mnist:sub3}\includegraphics[width=.525\columnwidth]{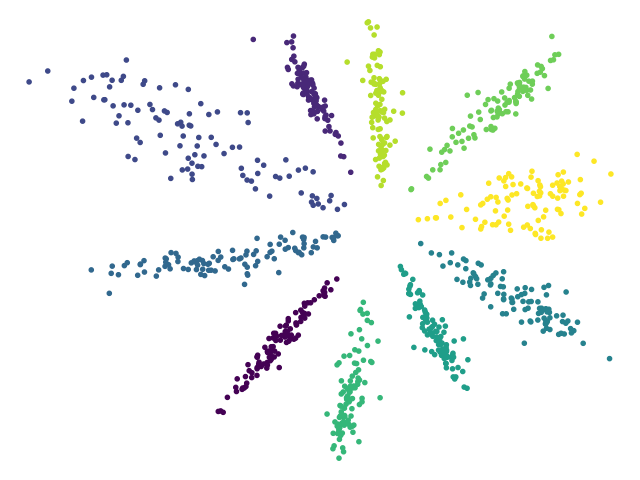}}
	\subfloat[]{\label{fig:mnist:sub4}\includegraphics[width=.525\columnwidth]{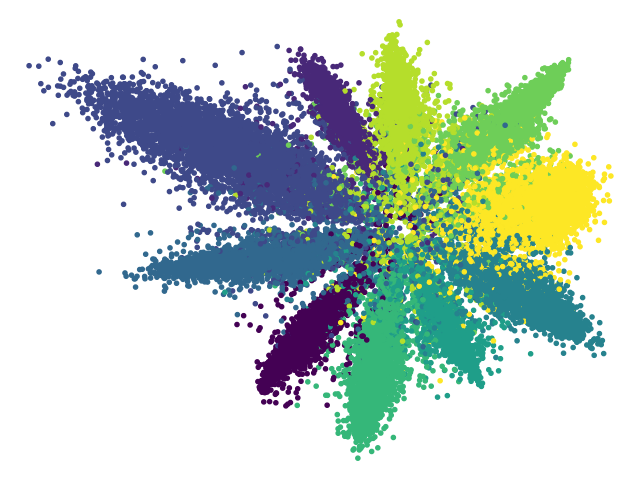}}
	\caption{Feature distribution on MNIST. First, LeNet was trained by labeled data. \protect\subref{fig:mnist:sub1} shows the the feature distribution of labeled images. Points with the same color belong to the same class. \protect\subref{fig:mnist:sub2} shows the feature distribution of both labeled and unlabeled images. Then, we used LeNet as the backbone network and trained the D2 framework. After training, \protect\subref{fig:mnist:sub3} and \protect\subref{fig:mnist:sub4} show the feature distribution of labeled images and all images, respectively. This figure needs to be viewed in color.}
	\label{fig:mnist}
\end{figure*}

Now, we use a toy example to explain how the D2 framework works intuitively. Inspired by \cite{TCP}, we use the LeNet~\cite{LeNet} as backbone structure and add two FC layers, in which the first FC layer learns a 2-D feature and the second FC layer projects the feature onto the class space. The network was trained on MNIST. Note that MNIST has 50000 images for training. We only used 1000 images as labeled images to train the network. Figure~\ref{fig:mnist:sub1} depicts the 2-D feature distribution of these 1000 images. We observe that features belonging to the same class will cluster together. Figure~\ref{fig:mnist:sub2} shows the feature distribution of both these 1000 labeled and other 49000 unlabeled images. Although the network did not train on the unlabeled images, features belonging to the same class are still roughly clustered.

Pseudo-labels in our D2 framework are probability distributions and initialized by network predictions. As Figure~\ref{fig:mnist:sub2} shows, features near the cluster center will have confident pseudo-labels and can be learned safely. However, features at the boundaries between clusters will have a pseudo-label whose corresponding distribution among different classes is flat rather than sharp. By training D2, the network will learn confident pseudo-labels first. Then it is expected that uncertain pseudo-labels will become more and more precise and confident by optimization. At last, each cluster will become more compact and the boundaries between different classes' features will become clear.  Figure~\ref{fig:mnist:sub4} depicts the feature distribution of all images after D2 training. Because the same class features of unlabeled images get closer, the same class features of labeled images will also get closer (cf. Figure~\ref{fig:mnist:sub3}). That is how unlabeled images help the training in our D2 framework.

\subsection{Repetitive reprediction}

Although D2 has worked well in practice (cf. Table~\ref{tab:ablation-result-table} column a), there are still some shortcomings in it. We will discuss two major ones. To mitigate these problems and further boost the performance, we propose a simple but effective strategy, repetitive reprediction (R2), to improve the D2 framework.

First, we expect pseudo-labels can become more confident along with D2's learning process. Unfortunately, we observed that more and more pseudo-labels become flat during training. Below, we prove Theorem~\ref{thm:2} to explain why this adverse effect happens.
\begin{thm}
	\label{thm:2}
	Suppose D2 is trained by SGD with the loss function $\mathcal{L}=\alpha\mathcal{L}_c + \beta\mathcal{L}_e$. If $\tilde{p}_n=\exp(-\frac{\mathcal{L}}{\alpha})\left(\hat{p}_n\right)^{1-\frac{\beta}{\alpha}}$, we must have $\tilde{p}_n\leq \hat{p}_n$. 
\end{thm}

\begin{proof}
	First, according to the loss function we defined, we have 
	\begin{align}
	\mathcal{L}
	&=\alpha\sum_{j=1}^{N}\hat{p}_j\left[\log(\hat{p}_j)-\log(\tilde{p}_j)\right]-\beta\sum_{j=1}^{N}\hat{p}_j\log(\hat{p}_j)\\
	&\geq-\beta\sum_{j=1}^{N}\hat{p}_j\log(\hat{p}_j) 
	\geq-\beta\sum_{j=1}^{N}\hat{p}_j\log(\hat{p}_n) \\
	&=-\beta\log(\hat{p}_n)\,,
	\end{align}
	where $\hat{p}_n$ is the largest value in $\{\hat{p}_1,\hat{p}_2,\dots,\hat{p}_N\}$. Then, from $\tilde{p}_n=\exp\left(-\frac{\mathcal{L}}{\alpha}\right)\hat{p}_n^{1-\frac{\beta}{\alpha}}$ and $\mathcal{L}\geq -\beta\log(\hat{p}_n)$, we have
	\begin{equation}
	\tilde{p}_n
	=\exp\left(-\frac{\mathcal{L}}{\alpha}\right)\hat{p}_n^{1-\frac{\beta}{\alpha}}
	\leq\exp\left(\frac{\beta\log(\hat{p}_n)}{\alpha}\right)\hat{p}_n^{1-\frac{\beta}{\alpha}}
	=\hat{p}_n\,.
	\end{equation}
\end{proof}

From Theorem~\ref{thm:1}, we get $\tilde{p}_n\rightarrow\exp(-\frac{\mathcal{L}}{\alpha})\left(\hat{p}_n\right)^{1-\frac{\beta}{\alpha}}$,  where $\hat{\mathbf{p}}$ gets the largest value at $\hat{p}_n$. And Theorem~\ref{thm:2} tells us if $\tilde{p}_n=\exp(-\frac{\mathcal{L}}{\alpha})\left(\hat{p}_n\right)^{1-\frac{\beta}{\alpha}}$ then $\tilde{p}_n$ will be smaller than $\hat{p}_n$. Because $\tilde{\mathbf{p}}$ and $\hat{\mathbf{p}}$ are probability distributions, if $\tilde{\mathbf{p}}$ and $\hat{\mathbf{p}}$ get their largest value at $n$, $\tilde{\mathbf{p}}$ is more flat than $\hat{\mathbf{p}}$ when $\tilde{p}_n\leq \hat{p}_n$. That is, along with the training of D2, there is a tendency that pseudo-labels will be more flat than the network predictions.

Second, we find an unsolicited bias in the D2 framework. From the updating formula, we can get 
\begin{align}
\sum_{i=1}^{N}\tilde{y}_i
&\leftarrow
\sum_{i=1}^{N}\tilde{y}_i
-\lambda\alpha\sum_{i=1}^{N}\sigma\left(\tilde{\mathbf{y}}\right)_i
+\lambda\alpha\sum_{i=1}^{N}\sigma\left(\hat{\mathbf{y}}\right)_i \\
&=\sum_{i=1}^{N}\tilde{y}_i-\lambda\alpha+\lambda\alpha =\sum_{i=1}^{N}\tilde{y}_i\,.
\end{align}
That is, $\sum_{i=1}^{N}\tilde{y}_i$ will \emph{not} change after initialization. Although we define $\tilde{\mathbf{y}}$ as the variable which is not constrained, the softmax function and SGD set an equality constraint for it. On the other hand, in practice, $\sum_{i=1}^{N}\hat{y}_i$ become more and more concentrated. Later, we will use an ablation study to demonstrate this bias is harmful.

We propose a repetitive reprediction (R2) strategy to overcome these difficulties, which repeatedly perform repredictions (i.e., using the prediction $\tilde{\mathbf{y}}$ to re-initialize the pseudo-labels $\tilde{\mathbf{y}}$ several times) during training D2. The benefits of R2 are two-fold. First, we want to make pseudo-labels confident. According to our analysis, the network predictions are sharper than pseudo-labels when the algorithm converges. So repredicting pseudo-labels can make them sharper.
Second, $\sum_{i=1}^{N}\tilde{y}_i$ will not change during D2 training. Reprediction can reduce the impact of this bias.
Furthermore, the validation accuracy often increase during training. A repeated reprediction can make pseudo-labels more accurate than that of the last reprediction. 

Apart from the repredictions, we also reduce the learning rate to boost the performance. If the D2 framework is trained by a fixed learning rate as in \cite{PENCIL}, the loss $\mathcal{L}$ did not descend in experiments. Reducing the learning rate can make the loss descend. We can get some benefits from a lower loss. On one hand, $\mathcal{L}_c$ is the KL divergence between pseudo-labels and the network predictions. Minimizing this term makes pseudo-labels as sharp as the network predictions. On the other hand, minimizing $\mathcal{L}_e$ can decrease the entropy of network predictions. So when it comes to next reprediction, pseudo-labels will be more confident according to sharper predictions. 

Finally, repredicting pseudo-labels frequently is harmful for performance. By using the R2 strategy every epoch, the network predictions and pseudo-labels are always the same and D2 cannot optimize pseudo-labels anymore. In CIFAR-10 experiments, we repredict pseudo-labels every 75 epochs and reduce the learning rate after each reprediction. Using the R2 strategy can make pseudo-labels more confident at the end of training.

\subsection{The overall R2-D2 algorithm}

Now we propose the overall R2-D2 algorithm. The training can be divided into three stages. In the first stage, we only use labeled images to train the backbone network with cross entropy loss as in common network training. In the second stage, we use the backbone network trained in the first stage to predict pseudo-labels for unlabeled images. Then we use D2 to train the network and optimize pseudo-labels together. It is expected that this stage can boost the network performance and make pseudo-labels more precise. But according to our analysis, it is not enough to train D2 by only one stage. With the R2 strategy, D2 will be repredicted and trained for several times. In the third stage, the backbone network is finetuned by all images whose labels come from the second stage. For unlabeled images, we pick the class which has the maximum value in pseudo-labels and use the cross entropy loss to train the network. And pseudo-labels are not updated anymore. For labeled images, we use their groundtruth labels.

In general, R2-D2 is a simple method. It requires only one single network (versus two in Mean Teacher and the loss function consists of two terms (versus three in Mean Teacher). The training processes in different stages are identical (share the same code), just need to change the value of two switch variables.

\section{Experiments}

In this section, we use four datasets to evaluate our algorithm: ImageNet~\cite{imagenet}, CIFAR-100~\cite{cifar}, CIFAR-10~\cite{cifar}, SVHN~\cite{SVHN}. We first use an ablation study to investigate the impact of the R2 strategy. We then report the results on these datasets to compare with state-of-the-arts. All experiments were implemented using the PyTorch framework and run on a computer with TITAN Xp GPU.

\subsection{Implementation details}

Note that we trained the network using stochastic gradient descent with Nesterov momentum 0.9 in all experiments. We set $\alpha = 0.1$, $\beta=0.03$ and $\lambda=4000$ on \emph{all} datasets, which shows the robustness of our method to these hyperparameters. Other hyperparameters (e.g., batch size, learning rate, and weight decay) were set according to different datasets. 

\textbf{ImageNet} is a large-scale dataset with natural color images from 1000 categories. Each category typically has 1300 images for training and 50 for evaluation. Following the prior work~\cite{DCT_2018_ECCV,Stochastic_Transformations,VAE,Mean_teacher}, we uniformly choose 10\% data from training images as labeled data. That means there are 128 labeled data for each category. The rest of training images are considered as unlabeled data. We test our model on the validation set. The backbone network is ResNet-18~\cite{ResNet}.

\textbf{CIFAR-100} contains $32\times 32$ natural images from 100 categories. There are 50000 training images and 10000 testing images in CIFAR-100. Following \cite{Temporal_Ensembling,DCT_2018_ECCV,cvpr2019_pseudo_label}, we use 10000 images (100 per class) as labeled data and the rest 40000 as unlabeled data. We report the error rates on the testing images. The backbone network is ConvLarge~\cite{Temporal_Ensembling}. 

\textbf{CIFAR-10} contains $32\times 32$ natural images from 10 categories. Following \cite{Temporal_Ensembling,VAT,Mean_teacher,DCT_2018_ECCV,HybridNet}, we use 4000 images (400 per class) from 50000 training images as labeled data and the rest images as unlabeled data. We report the error rates on the full 10000 testing images. The backbone network is Shake-Shake~\cite{shakeshake}.

\textbf{SVHN} consists of $32\times 32$ house number images belonging to 10 classes. The category of each image is the centermost digit. There are 73257 training images and 26032 testing images in SVHN. Following \cite{Temporal_Ensembling,Mean_teacher,VAT,DCT_2018_ECCV}, we use 1000 images (100 per class) as labeled data and the rest 72257 training images as unlabeled data. The backbone network is ConvLarge~\cite{Temporal_Ensembling}.

\subsection{Ablation studies}

\begin{table}
	\caption{Ablation studies when using different strategies to train our end-to-end framework.}
	\label{tab:ablation-result-table}
	\centering
	\small
	\begin{tabular}{cccccc}
		\toprule
		& a & b & c & d & e \\
		\midrule
		The 2nd stage 			& $\checkmark$ 	& $\checkmark$  & $\checkmark$ & $\checkmark$ &$\checkmark$  \\
		Repeat the 2nd stage	&				& $\checkmark$  & $\checkmark$ & $\checkmark$ &$\checkmark$  \\
		Reprediction			&  				& 				& $\checkmark$ & 			 &$\checkmark$  \\
		Reducing LR 			&  				& 				&			  & $\checkmark$ &$\checkmark$  \\
		Error rates (\%) 		& 6.71 			& 6.37 			& 6.23		   & 5.94		 & 5.78 \\
		\bottomrule
	\end{tabular}
\end{table}

Now we validate our framework by an ablation study on CIFAR-10 with the Shake-Shake backbone and 4000 labeled images. All experiments used the same data splits and ran once. And they all used the first stage to initialize D2 and the third stage to finetune the network. Table~\ref{tab:ablation-result-table} presents the results and the error rates are produced by the last epoch of the third stage. Different columns denote using different strategies to train D2 in the second stage.
First, without R2 (column a), the error rate of a basic D2 learning is $6.71\%$, which is already competitive with state-of-the-arts. Next, we repeated the second stage without reprediction or reducing learning rate (column b). That means the network is trained by the first stage, the second stage, repeat the second stage, and the third stage. This network achieved a $6.37\%$ error rate, which demonstrates training D2 for more epochs can boost performance and the network will not overfit easily. Repeating the second stage with reprediction (column c) could make the error rate even lower, to $6.23\%$. But, without reducing the learning rate, $\mathcal{L}$ did not decrease. On the other hand, repeating the second stage and reducing the learning rate (column d) can get better results ($5.94\%$). However, only reducing the learning rate cannot remove the impact of the equality constraint bias. At last, applying both strategies (column e) improved the results by a large margin to $5.78\%$.

\begin{table}
	\caption{Ablation studies when using different $\alpha$ to train our end-to-end framework. ($\beta=0.03$)}
	\label{tab:ablation-alpha}
	\centering
	\small
	\begin{tabular}{cccccc}
		\toprule
		$\alpha$ & 0.1 & 0.2 & 0.3 & 0.4 & 0.5 \\
		\midrule
		Error rates (\%) & 5.78 & 5.44 & 5.81 & 5.90 & 6.11 \\
		\bottomrule
	\end{tabular}
\end{table}

\begin{table}
	\caption{Ablation studies when using different $\beta$ to train our end-to-end framework. ($\alpha=0.1$)}
	\label{tab:ablation-beta}
	\centering
	\small
	\begin{tabular}{cccccc}
		\toprule
		$\beta$ & 0.01 & 0.02 & 0.03 & 0.04 & 0.05 \\
		\midrule
		Error rates (\%) & 5.62 & 5.75 & 5.78 & 5.83 & 5.76 \\
		\bottomrule
	\end{tabular}
\end{table}

Table~\ref{tab:ablation-alpha} presents the results with different $\alpha$. $\beta$ is $0.03$ in all experiments. We find setting $\alpha=0.2$ will achieve a better performance and a large $\alpha$ may degrade the performance. Table~\ref{tab:ablation-beta} shows the results with different $\beta$ when setting $\alpha=0.1$. Compared with $\alpha$, our method is robust to $\beta$. The highest error rate is $5.83\%$ and the lowest error rate is $5.62\%$. There is only a roughly $0.2\%$ difference between them. Overall, R2-D2 is robust to these hyperparameters. And when apply R2-D2, we suggest that $\alpha=0.1, \beta=0.03$ is a safe starting point to tune these hyperparameters. All experiments in the rest of our paper used $\alpha=0.1, \beta=0.03$. Please note that we did not carefully tune these hyperparameters. Error rates of R2-D2 may be lower than those reported in this paper if we tune them carefully.

\begin{table}
	\caption{Ablation studies when using different $\mathcal{L}_c$ to train our end-to-end framework. ($\alpha=0.1,\beta=0.03$)}
	\label{tab:ablation-loss}
	\centering
	\small
	\begin{tabular}{cccccc}
		\toprule
		$\mathcal{L}_c$ & $KL(\hat{\mathbf{p}}||\tilde{\mathbf{p}})$ & $KL(\tilde{\mathbf{p}}||\hat{\mathbf{p}})$ & $\|\tilde{\mathbf{p}} - \hat{\mathbf{p}}\|_2^2$ \\
		\midrule
		Error rates (\%) & 5.78 & 8.06 & 6.35 \\
		\bottomrule
	\end{tabular}
\end{table}

Table~\ref{tab:ablation-loss} shows the results with different $\mathcal{L}_c$. Note that our loss function is defined as $\mathcal{L}=\alpha\mathcal{L}_c + \beta\mathcal{L}_e$. The loss function determines how the network parameters and pseudo-labels update. That means different $\mathcal{L}_c$ result in different updating formulas of pseudo-labels. The default $\mathcal{L}_c$ is $KL(\hat{\mathbf{p}}||\tilde{\mathbf{p}})$ and the updating formula is Equation~\ref{updating-formula}. When set $\mathcal{L}_c=KL(\tilde{\mathbf{p}}||\hat{\mathbf{p}})$, the gradients of $\mathcal{L}$ with respect to $\tilde{\mathbf{y}}_n$ is 
\begin{equation}
\frac{\partial\mathcal{L}}{\partial\tilde{y}_n}
=\alpha
\tilde{\mathbf{p}}_n
[\log \tilde{\mathbf{p}}_n
- \log \hat{\mathbf{p}}_n
- \sum_{k=1}^{N}\tilde{\mathbf{p}}_k(\log\tilde{\mathbf{p}}_k - \log\hat{\mathbf{p}}_k)]
\,,
\end{equation}
where $\hat{\mathbf{p}} = \sigma(\hat{\mathbf{y}})$ and $\tilde{\mathbf{p}} = \sigma(\tilde{\mathbf{y}})$. With $\mathcal{L}_c=\|\tilde{\mathbf{p}} - \hat{\mathbf{p}}\|_2^2$, the gradients of $\mathcal{L}$ with respect to $\tilde{\mathbf{y}}_n$ is 
\begin{equation}
\frac{\partial\mathcal{L}}{\partial\tilde{y}_n}
=2\alpha
\tilde{\mathbf{p}}_n
[\tilde{\mathbf{p}}_n
- \hat{\mathbf{p}}_n
- \sum_{k=1}^{N}\tilde{\mathbf{p}}_k(\tilde{\mathbf{p}}_k - \hat{\mathbf{p}}_k)]
\,.
\end{equation}
The experimental results demonstrates superior performance of R2-D2 with $\mathcal{L}_c = KL(\hat{\mathbf{p}}||\tilde{\mathbf{p}})$. It obtains $2.28\%$ lower error rate than $KL(\tilde{\mathbf{p}}||\hat{\mathbf{p}})$ and $0.57\%$ lower error rate than $\|\tilde{\mathbf{p}} - \hat{\mathbf{p}}\|_2^2$.

\subsection{Results on ImageNet}

\begin{table*}[!ht]
	\caption{Error rates (\%) on the validation set of ImageNet benchmark with 10\% images labeled. ``-'' means that the original papers did not report the corresponding error rates.}
	\label{tab:ImageNet-result-table}
	\centering
	\small
	\begin{tabular}{c|llccc}
		\toprule
		&
		Method						& Backbone		&	\#Param	&	Top-1	&	Top-5	\\
		\midrule
		\multirow{2}*{Supervised} &
		100\% Supervised			
		&	ResNet-18	&	11.6M	&	30.43	&  10.76	\\
		&
		10\% Supervised	
		&	ResNet-18	&	11.6M	&	52.23	&	27.54	\\
		\midrule
		\multirow{5}*{Semi-supervised} &
		Stochastic Transformations
		&	AlexNet		&	61.1M	&     -		& 39.84		\\
		&							
		VAE with 10\% Supervised
		&	Customized	&	30.6M	&	51.59	&	35.24	\\
		&							
		Mean Teacher				
		&	ResNet-18	&	11.6M	&	49.07	&  23.59	\\
		&							
		Dual-View Deep Co-Training
		&	ResNet-18	&	11.6M	&	46.50	&	22.73	\\	
		&
		R2-D2					
		&	ResNet-18	&	11.6M	&\textbf{41.55}&\textbf{19.52}\\
		\bottomrule
	\end{tabular}
\end{table*}

\begin{table*}
	\caption{Error rates (\%) on CIFAR-100 benchmark with 10000 images labeled.}
	\label{tab:CIFAR-100-result-table}
	\centering
	\small
	\begin{tabular}{c|lcr@{.}l}
		\toprule
		&Method							&	Backbone	&	\multicolumn{2}{c}{Error rates (\%)}		\\
		\midrule
		\multirow{2}*{Supervised} &
		
		100\% Supervised				&	ConvLarge		&	$26$&$42 \pm 0.17$ 	\\
		&
		Using 10000 labeled images only	&	ConvLarge		&	$38$&$36 \pm 0.27$	\\
		\midrule
		\multirow{6}*{Semi-supervised} &		
		Temporal Ensembling				&	ConvLarge		&	$38$&$65 \pm 0.51$	\\
		&
		LP								&	ConvLarge		&	$38$&$43 \pm 1.88$	\\
		&
		Mean Teacher					&	ConvLarge		&	$36$&$08 \pm 0.51$	\\
		&
		LP + Mean Teacher				&	ConvLarge		&	$35$&$92 \pm 0.47$	\\
		&
		DCT								&	ConvLarge		&	$34$&$63 \pm 0.14$	\\
		&
		R2-D2 							&	ConvLarge		&$\textbf{32}$&$\textbf{87}\pm\textbf{0.51}$\\
		\bottomrule
	\end{tabular}
\end{table*}

Table~\ref{tab:ImageNet-result-table} shows our results on ImageNet with 10\% labeled samples. The setup followed that in \cite{DCT_2018_ECCV}. The image size in training and testing is $224\times 224$. For the fairness of comparisons, the error rate is from single model without ensembling. We use the result of the last epoch. Our experiment is repeated three times with different random subsets of labeled training samples. The Top-1 error rates are $41.64$, $41.35$, and $41.65$, respectively. The Top-5 error rates are $19.53$, $19.60$, and $19.44$, respectively. R2-D2 achieves significantly lower error rates than Stochastic Transformations \cite{Stochastic_Transformations} and VAE~\cite{VAE}, albeit they used the larger input size $256\times 256$. With the same backbone and input size, R2-D2 obtains roughly $5\%$ lower Top-1 error rate than that of DCT~\cite{DCT_2018_ECCV} and $7.5\%$ lower Top-1 error rate than that of Mean Teacher~\cite{Mean_teacher}. R2-D2 outperforms the previous state-of-the-arts by a large margin. The performances of Mean Teacher~\cite{Mean_teacher} with ResNet-18~\cite{ResNet} is quoted from \cite{DCT_2018_ECCV}.

\subsection{Results on CIFAR-100}

Table~\ref{tab:CIFAR-100-result-table} presents experimental results on CIFAR-100 with 10000 labeled samples. All methods used ConvLarge for fairness of comparisons and did not use ensembling. The error rate of R2-D2 is the average error rate of the last epoch over five random data splits. The results of 100\% Supervised is quoted from \cite{Temporal_Ensembling}. Using 10000 labeled images achieved $38.36\%$ error rates in our experiments. With unlabeled images, R2-D2 produced a $32.87\%$ error rate, which is lower than others (e.g., Temporal Ensembling, LP~\cite{cvpr2019_pseudo_label}, Mean Teacher, LP + Mean Teacher~\cite{cvpr2019_pseudo_label}, and DCT). The performances of Mean Teacher~\cite{Mean_teacher} is quoted from \cite{cvpr2019_pseudo_label}.

\subsection{Results on CIFAR-10}

\begin{table}
	\caption{Error rates (\%) on CIFAR-10 benchmark with 4000 images labeled.}
	\label{tab:CIFAR-10-result-table}
	\centering
	\small
	\begin{tabular}{lcr@{.}l}
		\toprule
		Method			&	Backbone	&	\multicolumn{2}{c}{Error rates (\%)}	\\
		\midrule
		
		100\% Supervised				&	Shake-Shake		&	$2$&$86$			\\
		
		Only 4000 labeled images 		&	Shake-Shake		&	$14$&$90\pm0.28$	\\
		
		\midrule
		
		Mean Teacher					&	ConvLarge		&	$12$&$31\pm0.28$	\\
		
		Temporal Ensembling				&	ConvLarge		&	$12$&$16\pm0.24$	\\
		
		VAT+EntMin						&	ConvLarge		&	$10$&$55\pm0.05$	\\
		
		DCT with 8 Views				&	ConvLarge		&	$8$&$35\pm0.06$		\\
		
		Mean Teacher					&	Shake-Shake		&	$6$&$28\pm0.15$		\\
		
		HybridNet						&	Shake-Shake		&	$6$&$09$			\\
		
		R2-D2 							&	Shake-Shake		&$\textbf{5}$&$\textbf{72}\pm\textbf{0.06}$	\\
		\bottomrule
	\end{tabular}
\end{table}

We evaluated the performance of R2-D2 on CIFAR-10 with 4000 labeled samples. Table~\ref{tab:CIFAR-10-result-table} presents the results. Following \cite{Mean_teacher,HybridNet}, we used the Shake-Shake network~\cite{shakeshake} as the backbone network. Overall, using Shake-Shake backbone network can achieves lower error rates than using ConvLarge. Our experiment was repeated five times with different random subsets of labeled training samples. We used the test error rates of the last epoch. After the first stage, the backbone network produced the error rates 14.90\%, which is our baseline using 4000 labeled samples. With the help of unlabeled images, R2-D2 obtains an error rate of $5.72\%$.Compared with Mean Teacher~\cite{Mean_teacher} and HybridNet~\cite{HybridNet}, R2-D2 achieves lower error rate and produces state-of-the-art results. 

\subsection{Results on SVHN}

\begin{table}
	\caption{Error rates (\%) on SVHN benchmark with 1000 images labeled.}
	\label{tab:SVHN-result-table}
	\centering
	\small
	\begin{tabular}{lcr@{.}l}
		\toprule
		Method						&	Backbone	&	\multicolumn{2}{c}{Error rates (\%)}		\\
		\midrule

		100\% Supervised			&	ConvLarge	&	$2$&$88\pm0.03$	\\

		Only 1000 labeled images 	&	ConvLarge	&	$11$&$27\pm0.85$ \\
		\midrule
		Temporal Ensembling			&	ConvLarge	&	$4$&$42\pm0.16$	\\
		
		VAdD (KL)					&	ConvLarge	&	$4$&$16\pm0.08$ \\

		Mean Teacher				&	ConvLarge	&	$3$&$95\pm0.19$	\\	

		VAT+EntMin					&	ConvLarge	&	$3$&$86\pm0.11$	\\
		
		VAdD (KL) + VAT 			&	ConvLarge	&	$3$&$55\pm0.05$ \\

		DCT with 8 Views			&	ConvLarge	&$\textbf{3}$&$\textbf{29}\pm\textbf{0.03}$\\

		R2-D2						&	ConvLarge	&	$3$&$64\pm0.20$	\\
		\bottomrule
	\end{tabular}
\end{table}

We tested R2-D2 on SVHN with 1000 labeled samples. The results are shown in Table~\ref{tab:SVHN-result-table}.
Following previous works~\cite{Temporal_Ensembling,Mean_teacher,VAT,DCT_2018_ECCV}, we used the ConvLarge network as the backbone network. The result we report is average error rate of the last epoch over five random data splits. On this task, the gap between 100\% supervised and many SSL methods (e.g., VAT+EntMin~\cite{VAT}, VAdD (KL)+VAT~\cite{VAdD}, Deep Co-Training~\cite{DCT_2018_ECCV}, and R2-D2) is less than 1\%. Only Deep Co-Training with 8 Views~\cite{DCT_2018_ECCV} and VAdD (KL)+VAT slightly outperform R2-D2. Compared with other methods (e.g., Temporal Ensembling, Mean Teacher, and VAT, R2-D2 produces a lower error rate. Note that on the large-scale ImageNet, R2-D2 significantly outperformed Deep Co-Training. VAdD have not be evaluated on ImageNet in their paper.

\section{Conclusion}

In this paper, we proposed R2-D2, a method for semi-supervised deep learning. D2 uses label probability distributions as pseudo-labels for unlabeled images and optimizes them during training. Unlike previous SSL methods, D2 is an end-to-end framework, which is independent of the backbone network and can be trained by back-propagation. Based on D2, we give a theoretical support for using network predictions as pseudo-labels. However, pseudo-labels will become flat during training. We analyzed this problem both theoretically and experimentally, and proposed the R2 remedy for it. At last, we tested R2-D2 on different datasets. The experiments demonstrated superior performance of our proposed methods. On large-scale dataset ImageNet, R2-D2 achieved about $5\%$ lower error rates than that of previous state-of-the-art. In the future, we will explore the combination of unsupervised feature learning and semi-supervised learning.

\section{Acknowledgments}

This work is supported by the National Natural Science Foundation of China (61772256, 61921006).

\bibliography{AAAI-WangG}
\bibliographystyle{aaai}
\end{document}